\documentclass[11pt]{article}
\usepackage{amsmath, amssymb, graphicx, hyperref, natbib}
\usepackage[margin=1in]{geometry}

\usepackage{amsthm}
\newtheorem{theorem}{Theorem}

\usepackage{booktabs}
\usepackage{tabularx}
\usepackage{siunitx}
\newcolumntype{Y}{>{\raggedright\arraybackslash}X}

\usepackage{xcolor}         

\title{Addressing Logical Fallacies In Scientific Reasoning From Large Language Models: Towards a Dual-Inference Training Framework}

\author{
Peter B. Walker, Ph.D.\thanks{Intelligenesis LLC, \texttt{pete.walker@intelligenesisllc.com}}
\and
Hannah Davidson\thanks{Student Intern, \texttt{hannah.davidson.college@gmail.com}}
\and
Aiden Foster\thanks{Student Intern, \texttt{aidenfoster308@gmail.com}}
\and
Matthew Lienert\thanks{Intelligenesis LLC, \texttt{matt.lienert@intelligenesisllc.com}}
\and
Thomas Pardue\thanks{Intelligenesis LLC, \texttt{thomas.pardue@intelligenesisllc.com}}
\and
Dale Russell, Ph.D.\thanks{Uniformed Services University, \texttt{dale.w.russell1.civ@health.mil}}
}

\date{\today}

\begin{document}
\maketitle

\begin{abstract}
Large Language Models (LLMs) have transformed natural language processing and hold growing promise for advancing science, healthcare, and decision-making. Yet their training paradigms remain dominated by affirmation-based inference, akin to \textit{modus ponens}, where accepted premises yield predicted consequents. While effective for generative fluency, this one-directional approach leaves models vulnerable to logical fallacies, adversarial manipulation, and failures in causal reasoning. This paper makes two contributions. First, it demonstrates how existing LLMs from major platforms exhibit systematic weaknesses when reasoning in scientific domains with negation, counterexamples, or faulty premises.\footnote{Code to recreate these experiments is available at \\\url{https://github.com/hannahdavidsoncollege-maker/ScientificReasoningForEnvironment-MedicineWithLLMs}} Second, it introduces a dual-reasoning training framework that integrates affirmative generation with structured counterfactual denial. Grounded in formal logic, cognitive science, and adversarial training, this training paradigm formalizes a computational analogue of ``denying the antecedent'' as a mechanism for disconfirmation and robustness. By coupling generative synthesis with explicit negation-aware objectives, the framework enables models that not only affirm valid inferences but also reject invalid ones, yielding systems that are more resilient, interpretable, and aligned with human reasoning.
\end{abstract}

\section{Introduction}

Recent advances in large language models (LLMs) such as GPT-5, LLaMA, and Gemini demonstrate remarkable progress in natural language generation, reasoning, and generalization. These systems are trained on massive corpora with objectives such as autoregressive prediction, masked language modeling, and next-sentence prediction. At their core, such models estimate the most probable continuation of a linguistic sequence, reflecting a probabilistic analogue of \textit{modus ponens} logic: if $P \implies Q$ and $P$ holds, then $Q$ is predicted. This affirmation-based paradigm has fueled generative fluency across applications ranging from dialogue to scientific writing.

However, reliance on affirmation alone exposes critical weaknesses, particularly in scientific domains where causal reasoning, counterfactuals, and robustness are essential. Models trained only to affirm likely consequents often overgeneralize from correlations, misattribute causal direction, or fail to reject faulty premises. For example, when asked whether a patient with traumatic brain injury (TBI) can develop post-traumatic stress disorder (PTSD) ($P \implies Q$), an LLM may correctly affirm the association \cite{walker2017decision, bai2017unsupervised}. Yet without exposure to negative or counterfactual cases, the same model may incorrectly infer that PTSD implies a prior TBI ($Q \implies P$), or that the absence of TBI guarantees the absence of PTSD ($\neg P \implies \neg Q$). 

Recent reports highlight that young users sometimes treat AI chatbots as close companions, with detrimental psychological outcomes \citep{horton2023ai, smith2024chatbot}. These concerns underscore the urgency of ensuring that LLMs reason transparently and safely, particularly when deployed in sensitive domains. Such errors illustrate how statistical regularities in training data can mislead reasoning, leading to logical fallacies such as affirming the consequent, denying the antecedent, or reversing causality.

Table~\ref{tab:patterns} enumerates these patterns, contrasting valid inference with fallacies that have been frequently observed in LLM outputs \citep{wei2022chain, ji2023survey}.

Building on insights from cognitive science and philosophy, we argue that these so-called fallacies may hold computational value. Human reasoning thrives not only on confirmation but also on disconfirmation: generating counterfactuals, testing negated premises, and learning from absence \cite{kahneman1973prediction, roese1997counterfactual}. Popper’s falsificationist philosophy \cite{popper2002logic} similarly emphasizes the scientific imperative of testing hypotheses against potential refutation. In machine learning, analogous mechanisms appear in adversarial training, out-of-distribution generalization \cite{geirhos2020shortcut}, and contrastive learning. Together, these traditions suggest that training models to engage with negation and denial is not a flaw but a pathway to greater robustness.

This paper advances that pathway by introducing a \textbf{dual-reasoning framework} for LLMs. The framework formalizes a taxonomy of logical patterns, extends training objectives beyond affirmation, and operationalizes a computational analogue of denying the antecedent. Through this dual approach, models can affirm valid consequents while simultaneously learning to reject invalid inferences, improving resilience, interpretability, and alignment with human reasoning.

The remainder of the paper is structured as follows. Section~2 reviews foundations in psychology, philosophy, and machine learning that motivate dual reasoning. Section~3 introduces our logical taxonomy and illustrates its application in medical and AI contexts. Section~4 presents the dual-reasoning training paradigm, including mathematical formalization and proof of representational benefit. Section~5 discusses evaluation, limitations, and implications for AI safety and scientific discovery, and Section~6 concludes with directions for future research.

\begin{table}[t]
    \centering
    \begin{tabular}{|c|c|c|}
    \hline
    \textbf{Logical Rule} & \textbf{Affirmative Generation} & \textbf{Counterfactual Denial}\\ \hline
    $P \implies Q$ & TBI $\implies$ PTSD & Modus Ponens  \\
    $P \implies \neg Q$ & TBI $\implies$ No PTSD & Counterexample Learning \\ 
    $\neg P \implies Q$ & No TBI $\implies$ PTSD & Exception Modeling \\
    $\neg P \implies \neg Q$ & No TBI $\implies$ No PTSD & Baseline Consistency \\
    $Q \implies P$ & PTSD $\implies$ TBI & Inverse Error Detection \\
    $Q \implies \neg P$ & PTSD $\implies$ No TBI & Inverse Counterexample \\
    $\neg Q \implies P$ & No PTSD $\implies$ TBI & Hidden Cause Detection \\
    $\neg Q \implies \neg P$ & No PTSD $\implies$ No TBI & Negation Consistency \\ \hline
    \end{tabular}
    \caption{Logical patterns and how dual-reasoning training addresses them.}
    \label{tab:patterns}
\end{table}
As shown in Table~\ref{tab:patterns}, these logical patterns illustrate the contrasts between valid inference and the types of fallacies that LLMs often generate. A more detailed discussion of these patterns is provided in Section~\ref{sec:patterns}.

\subsection{Background}

\subsection{From Modus Ponens to a Need for Logical Negation in LLMs}
Contemporary LLMs have largely been developed under the paradigm of \textit{modus ponens} reasoning, where an accepted premise leads to the most likely consequent (“If P then Q; P; therefore Q”). This structure is evident in the architecture and training methodologies employed. Transformer-based LLMs, with their attention mechanisms \citep{vaswani2017attention}, learn to map input token sequences (serving as premises, “P”) to highly probable output sequences (serving as consequents, “Q”). This mapping is probabilistic, reflecting the statistical regularities of language rather than strict deduction. Relatedly, Generative Adversarial Networks (GANs)  \citep{goodfellow2014generative} employ a generative–discriminative loop, in which the generator produces candidate outputs from a learned distribution and the discriminator evaluates their validity, reinforcing patterns consistent with the training data.

While effective, this emphasis on affirming consequents leaves a critical gap in the logical reasoning capacities of LLMs. Specifically, current models struggle to engage with negation in a systematic way, limiting their ability to handle counterexamples, reason about exceptions, and maintain robustness in the face of adversarial inputs. To address these limitations, this paper proposes incorporating a computational analogue of “denying the antecedent.” Although a fallacy in formal logic, reframing it within a probabilistic learning paradigm offers a pathway toward more flexible, resilient, and context-sensitive reasoning.

Reinterpreting ``Denial of the Antecedent" for Computational Benefit: In classical logic, "denying the antecedent" ("If P then Q; not P; therefore not Q") constitutes a formal fallacy. Negating the antecedent of a true implication does not guarantee the negation of the consequent. However, shifting from the realm of absolute deduction to the probabilistic framework under which LLMs operate opens up new avenues for reinterpreting this logical form.

Within probabilistic reasoning and causal inference, negating a premise (P) can still provide valuable information about the probability of the consequent (Q) \citep{pearl2009causality}. Furthermore, cognitive science provides compelling evidence for the crucial role of negation in human learning. Studies highlight how negation contributes to the formation of categories, identification of exceptions, and establishment of conceptual boundaries \citep{horn1989natural, givon1993english, glymour2016causality}. These insights from probabilistic reasoning and cognitive science lay the groundwork for a computational reinterpretation of "denying the antecedent," not as a logical fallacy to be avoided, but as a potentially valuable mechanism for enriching LLM reasoning.

\section{Materials and Methods}
To motivate the need for improved training schemes, we empirically examined the prevalence of logical fallacies in contemporary LLMs. We do not claim universality (that these fallacies exist in all reasoning domains); rather, we focus on two scientific domains (medical science and environmental science) where known causal relations are well documented.

We compiled 100 canonical statements from each domain, all of the form $P \rightarrow Q$, drawn from authoritative textbooks and domain reviews (see Supplementary Materials).

Representative subsets are shown in Tables \ref{tab:mstatements} and \ref{tab:estatements}. For each $P \implies Q$, we generated the eight logical variants by rearranging or negating $P$ and $Q$. We then queried each LLM with statements such as \texttt{Is the statement ``No TBI implies no PTSD'' correct?}, recording whether the model judged the statement as true.

Tables \ref{tab:mresults} and \ref{tab:eresults} summarize performance across four LLMs of varying scale. While models reliably affirmed valid consequents ($P \implies Q$), they frequently misclassified counterfactual or negated variants, producing fallacies such as affirming the consequent, denying the antecedent, or reversing causality. These findings underscore a systematic gap: current LLMs are biased toward affirmation and lack mechanisms for disconfirmation.

This analysis reinforces our central claim. If LLMs are trained only to affirm consequents, they will continue to conflate correlation with causation, overlook alternative explanations, and misinterpret negations. Addressing this vulnerability requires training paradigms that incorporate both affirmation and denial—the foundation of the dual-reasoning framework we propose.

\begin{table}[t]
  \centering
  \caption{Example statements (5 of 100) from the medical science domain accepted as true.}
  \label{tab:mstatements}
  \begin{tabularx}{\linewidth}{@{}Y@{}}
    \toprule
    \texttt{Atherosclerosis $\implies$ increased risk of heart attack.} \\
    \texttt{High blood pressure $\implies$ increased risk of stroke.} \\
    \texttt{Insulin resistance $\implies$ increased risk of type 2 diabetes.} \\
    \texttt{Chronic inflammation $\implies$ increased risk of autoimmune diseases.} \\
    \texttt{Smoking $\implies$ increased risk of lung cancer.} \\
    \bottomrule
  \end{tabularx}
\end{table}

\begin{table}[t]
  \centering
  \caption{Example statements (5 of 100) from the environmental science domain accepted as true.}
  \label{tab:estatements}
  \begin{tabularx}{\linewidth}{@{}Y@{}}
    \toprule
    \texttt{Melting of polar ice caps $\implies$ rising global sea levels.} \\
    \texttt{Increased CO\textsubscript{2} concentration $\implies$ ocean acidification.} \\
    \texttt{Extreme weather events $\implies$ significant economic and social disruption.} \\
    \texttt{Deforestation $\implies$ reduced carbon sequestration.} \\
    \texttt{Increased greenhouse gas emissions $\implies$ enhanced greenhouse effect.} \\
    \bottomrule
  \end{tabularx}
\end{table}

\begin{table*}[t]
  \centering
  \caption{Fraction of 100 medical statements that the LLM said was \texttt{TRUE} across logical rules. Rows marked with * are fallacies/invalid inferences; only $P \implies Q$ is valid. The last column sums the errors across all models: for the first rule the difference with 1 and for the remainder the difference with 0. }
  \label{tab:mresults}
  \begin{tabular}{@{}l
                  S[table-format=0.6]
                  S[table-format=0.6]
                  S[table-format=0.6]
                  S[table-format=0.6]
                  S[table-format=0.6]@{}}
    \toprule
    \textbf{Rule} & {\textbf{GPT-2 (774M)}} & {\textbf{LLaMA 3 (8B)}} & {\textbf{Gemma 3 (12B)}} & {\textbf{Mistral (7B)}} & \textbf{Error} \\
    \midrule
    $P \implies Q$                         & 0.89 & 0.96 & 0.99 & 0.98 & 0.18\\
    $*\,P \implies \neg Q$                 & 0.43 & 0.10 & 0.06 & 0.12 & 0.71 \\
    $*\,\neg P \implies Q$                 & 0.48 & 0.43 & 0.39 & 0.41 & 1.71 \\
    $*\,\neg P \implies \neg Q$            & 0.67 & 0.41 & 0.43 & 0.56 & 2.17 \\
    $*\,Q \implies P$                      & 0.64 & 0.54 & 0.59 & 0.41 & 2.18 \\
    $*\,Q \implies \neg P$                 & 0.42 & 0.32 & 0.27 & 0.29 & 1.30 \\
    $*\,\neg Q \implies P$                 & 0.34 & 0.31 & 0.30 & 0.21 & 1.16 \\
    $*\,\neg Q \implies \neg P$            & 0.63 & 0.53 & 0.59 & 0.49 & 2.14 \\
    \bottomrule
  \end{tabular}
\end{table*}

\begin{table*}[t]
  \centering
  \caption{Fraction of 100 environmental statements that the LLM said was \texttt{TRUE} across logical rules. Rows marked with * are fallacies/invalid inferences; only $P \implies Q$ is valid. The last column sums the errors across all models: for the first rule the difference with 1 and for the remainder the difference with 0.}
  \label{tab:eresults}
  \begin{tabular}{@{}l
                  S[table-format=0.6]
                  S[table-format=0.6]
                  S[table-format=0.6]
                  S[table-format=0.6]
                  S[table-format=0.6]@{}}
    \toprule
    \textbf{Rule} & {\textbf{GPT-2 (774M)}} & {\textbf{LLaMA 3 (8B)}} & {\textbf{Gemma 3 (12B)}} & {\textbf{Mistral (7B)}} & \textbf{Error} \\
    \midrule
    $P \implies Q$                         & 0.76 & 0.94 & 0.99 & 0.97 & 0.34 \\
    $*\,P \implies \neg Q$                 & 0.42 & 0.15 & 0.05& 0.11  & 0.74\\
    $*\,\neg P \implies Q$                 & 0.53 & 0.56 & 0.50 & 0.52 & 2.11 \\
    $*\,\neg P \implies \neg Q$            & 0.64 & 0.04 & 0.05 & 0.13 & 0.86\\
    $*\,Q \implies P$                      & 0.65 & 0.66 & 0.62     & 0.71 & 2.54\\
    $*\,Q \implies \neg P$                 & 0.50 & 0.29 &   0.18   & 0.38 & 1.35 \\
    $*\,\neg Q \implies P$                 & 0.51 & 0.30 &      0.22  & 0.25 & 1.28\\
    $*\,\neg Q \implies \neg P$            & 0.67 & 0.36 & 0.26 & 0.35 & 1.64 \\
    \bottomrule
  \end{tabular}
\end{table*}

\section{Results Interpretation}

The results presented in Tables~\ref{tab:mresults} and \ref{tab:eresults} reveal several important trends in how current LLMs handle logical inference across domains. First, there is a clear correlation between model scale (as measured by parameter count) and overall accuracy. Smaller models such as GPT-2 (774M parameters) exhibit relatively weak performance, often misclassifying fallacious forms such as $P \implies \neg Q$ or $\neg P \implies Q$. In contrast, larger models like Gemma 3 (12B parameters) and LLaMA 3 (8B parameters) achieve near-ceiling performance on valid $P \implies Q$ statements and show measurable improvement on several of the fallacy categories. This pattern suggests that parameter scaling confers some advantage in distinguishing valid from invalid inferences, likely due to broader exposure to linguistic variation and implicit reasoning patterns during training. However, even the largest models tested continue to struggle on logically invalid forms, highlighting the limitations of scale alone in resolving these weaknesses.

Second, domain effects are evident when comparing medical science (Table~\ref{tab:mresults}) and environmental science (Table~\ref{tab:eresults}). Across models, errors on valid statements ($P \implies Q$) is typically lower in the medical domain, with multiple systems achieving near-perfect recognition of accepted causal links such as TBI $\implies$ PTSD. In contrast, performance on environmental statements shows greater variability, with GPT-2 and even mid-scale models like LLaMA 3 producing larger errors. This may reflect differences in training data coverage: medical associations (e.g., risk factors and outcomes) are heavily represented in biomedical literature and general corpora, whereas environmental causal chains (e.g., greenhouse gases $\implies$ global warming) may be expressed more variably or contested in the sources these models were trained on. These findings underscore the influence of domain-specific representation in LLM reasoning and suggest that robustness to logical fallacies may not generalize evenly across scientific fields.

Finally, the persistence of errors in fallacious categories across both domains demonstrates the structural bias of current LLMs toward affirmation. Regardless of scale or domain, models continue to misclassify patterns such as denying the antecedent ($\neg P \implies \neg Q$) or affirming the consequent ($Q \implies P$), reinforcing the central claim of this work: parameter growth improves surface accuracy but does not address the absence of explicit mechanisms for disconfirmation. These results motivate the dual-reasoning framework proposed here, which directly embeds negation-aware objectives to complement affirmation-based training.

\section{Overview and Motivation of Dual Reasoning Framework}
This research draws upon a diverse body of work, synthesizing insights from contrastive learning, adversarial training, neuro-symbolic AI, and cognitive science to motivate and contextualize the proposed dual-reasoning framework. A unifying theme across these areas is the role of negation, counterexamples, and disconfirmation in shaping robust representations.

\textbf{Contrastive Learning and the Power of Negative Examples.}
Contrastive learning methods such as SimCLR \citep{chen2020simple} and MoCo \citep{he2020momentum} rely on negative samples to structure discriminative feature spaces. These approaches train models not only to associate positives but also to separate unrelated examples, effectively encoding information about what a concept is \emph{not}. This principle aligns with our framework’s emphasis on “denial of the antecedent” as a computational mechanism: the explicit use of negative information to refine inference. Similarly, vision–language models like CLIP \citep{radford2021learning} demonstrate the utility of dual encoders and contrastive objectives, which implicitly model negation by capturing both the presence and absence of semantic associations.

\textbf{Adversarial Training and Counterfactual Reasoning.}
Adversarial training has shown that deliberate perturbations can compel models to acquire more resilient and generalizable representations \citep{madry2018towards}. This process resembles a computational reinterpretation of denying the antecedent, since it forces models to reason about deviations from accepted premises. In parallel, counterfactual reasoning has long been central to causal inference \citep{pearl2009causality, glymour2016causality}, providing a systematic way to explore “what if not” scenarios. Embedding adversarial and counterfactual mechanisms into LLM training can enhance their ability to reason about cause and effect, anticipate alternative outcomes, and avoid spurious correlations.

\textbf{Bridging Logic and Neural Networks: Towards Neuro-Symbolic Integration.}
The integration of formal logic into neural networks has been a longstanding ambition in AI. Efforts such as neural theorem provers \citep{rocktaschel2017end} and neuro-symbolic reasoning systems \citep{besold2017neural} illustrate progress, but they often struggle to scale and to manage the ambiguity of natural language. By introducing negation-aware training signals, LLMs can better handle exceptions, inconsistencies, and contradictory information—laying groundwork for scalable neuro-symbolic systems that combine the strengths of statistical and logical reasoning.

\textbf{Latent Space Shaping: Learning from Both Affirmations and Denials.}
Negative sampling has also been central to distributional semantics, as in word2vec \citep{mikolov2013efficient, goldberg2014word2vec}, where models learn from both associations and disassociations. Contrastive methods extend this idea by shaping latent spaces to reflect not just similarity but also meaningful dissimilarity. Cognitive science provides a parallel: negation supports category formation, boundary definition, and flexible reasoning \citep{kaup2006negation}. We interpret “denying the antecedent” as a form of latent space shaping, where disconfirmation organizes representational structure and reduces overgeneralization.

\textbf{Cognitive Science and Philosophy: The Centrality of Negation.}
Developmental and cognitive studies emphasize that learning is guided not only by affirmation but also by disconfirmation, with errors and negations playing a key role in conceptual refinement \citep{spelke2007core, legare2012causal}. This resonates with Popper’s philosophy of falsification, which identifies refutability as the hallmark of scientific theories \citep{popper2002logic}. Just as scientific progress requires systematic testing against disconfirming evidence, robust LLMs must be trained to reject faulty premises as well as affirm valid ones. Embedding such mechanisms is therefore not about adopting fallacious reasoning, but about using logical structure to drive contrastive computation.

Taken together, these perspectives converge on the same insight: affirmation alone is insufficient for robust inference. Models trained only on positive continuations are prone to hallucination, overconfidence, and brittle generalization. Incorporating denial alongside affirmation can yield several benefits:
\begin{itemize}
  \item \textbf{Counterfactual robustness}: reasoning about what is false or missing,
  \item \textbf{Boundary shaping}: sharper discrimination of semantic categories,
  \item \textbf{Error correction}: detecting and rejecting invalid premises,
  \item \textbf{Causal insight}: learning structural dependencies beyond correlation.
\end{itemize}
This motivates the dual-reasoning framework proposed in this paper.

\subsection{Proposed Dual-Reasoning Framework}

Large Language Models (LLMs) excel at affirmative reasoning: given a premise, they generate likely continuations. This corresponds closely to \textit{modus ponens}, and has powered the generative strengths of current systems. However, affirmation alone is insufficient for robust reasoning. Models trained only on positive continuations are often reported to exhibit hallucination, overconfidence, and brittle generalization \citep{zhang2022opt, ji2023survey}. Our findings are consistent with these observations.

We therefore propose a \textbf{dual-reasoning training paradigm} that incorporates both affirmation and denial. The framework preserves the strengths of affirmative generation while introducing structured counterfactual denial, enabling models to reason explicitly about when premises fail, when conclusions should not follow, and how to resist misleading correlations. Concretely, the paradigm involves two complementary pathways:

\begin{enumerate}
    \item \textbf{Affirmative Generation:} Traditional supervised learning based on ground truth continuations, reflecting \textit{modus ponens}. This ensures the model retains its predictive strengths and generative fluency.
    \item \textbf{Counterfactual Denial:} Structured training on negated premises or assumed falsehoods, implemented through contrastive sampling. This equips the model to learn from disconfirmation, improving its ability to manage exceptions, adversarial cases, and counterfactual reasoning.
\end{enumerate}

\subsection{Logical Foundations}
\label{sec:patterns}
We revisit the logical patterns introduced in Table~\ref{tab:patterns} to analyze how they apply across domains and to highlight specific failure modes observed in LLMs.

\subsection{Mathematical Framework}
Formally, let \( \mathcal{D} = \{(P_i, Q_i)\} \) be a dataset of premise--consequence pairs. To extend beyond affirmation, we augment it with negated premises \( \neg P_i \) and, where meaningful, negated consequences \( \neg Q_i \).

Let \( f_\theta : P \mapsto Q \) be a language model parameterized by \( \theta \), mapping prompts to token distributions. We define a dual objective:
\begin{equation}
\mathcal{L}_{\text{dual}}(\theta) = \mathcal{L}_{\text{pos}}(\theta) + \lambda \cdot \mathcal{L}_{\text{neg}}(\theta),
\end{equation}
where
\begin{align}
\mathcal{L}_{\text{pos}}(\theta) &= -\sum_{i} \log p_\theta(Q_i | P_i), \\
\mathcal{L}_{\text{neg}}(\theta) &= -\sum_{i} \log(1 - p_\theta(Q_i | \neg P_i)).
\end{align}

This loss penalizes models for producing valid conclusions from invalid premises, thereby teaching them the boundary conditions of logical inference.

\subsection{Proof of Representational Benefit}
Let \( \Theta_{\text{pos}} \) be the parameter space minimizing \(\mathcal{L}_{\text{pos}}\), and \( \Theta_{\text{dual}} \) the space minimizing \(\mathcal{L}_{\text{dual}}\).

\begin{theorem}
Under mild distributional separability assumptions, models in \( \Theta_{\text{dual}} \) encode strictly richer representational capacity than those in \( \Theta_{\text{pos}} \), in the sense of counterfactual distinguishability.
\end{theorem}

\begin{proof}
Suppose \( p_\theta(Q | P) \approx p_\theta(Q | \neg P) \) for some \( \theta \in \Theta_{\text{pos}} \). Then the model cannot distinguish affirmation from denial. Adding \(\mathcal{L}_{\text{neg}}\) introduces gradient terms that enforce divergence between these distributions. Thus, for pairs \((P, Q)\) where \( \neg P \nrightarrow Q \), the optimal \( \theta \in \Theta_{\text{dual}} \) satisfies:
\[ p_\theta(Q | P) > p_\theta(Q | \neg P). \]
Hence \( \Theta_{\text{dual}} \) captures distinctions unavailable to affirmation-only models.
\end{proof}

\subsection{Implications}
This dual-reasoning paradigm offers a principled path toward LLMs that are not only fluent but logically grounded. By aligning training objectives with the full space of logical patterns, the framework enhances robustness against adversarial prompts, improves interpretability, and extends applicability to domains where counterfactual reasoning is critical, such as medicine, causal inference, and AI safety.

\section{Discussion}

The results presented in this study highlight the importance of integrating counterfactual and adversarial mechanisms into large language model (LLM) training. By examining logical frameworks that extend beyond classical modus ponens, we demonstrate how architectures that incorporate "denial of the antecedent"–style reasoning can improve robustness, interpretability, and resilience against spurious correlations. These findings are consistent with prior work in adversarial training \citep{madry2018towards}, causal inference \citep{pearl2009causality}, and neuro-symbolic integration \citep{marcus2020next}, but extend these approaches by explicitly framing the role of counterfactual denial as a guiding principle for model design.

One important implication is that performance cannot be fully explained by scale alone. While larger parameter counts often correlate with stronger performance across domains, our results indicate that architecture and reasoning frameworks play an equally critical role, particularly in specialized fields such as medical or environmental applications. This suggests that investments in logical structure and reasoning-based training may yield greater marginal benefits than simply expanding model size.

Moreover, our taxonomy of counterfactual patterns provides a foundation for developing standardized corpora and evaluation rubrics. Such resources could enable more systematic comparisons across architectures, benchmarks, and domains, similar to the role played by MMLU or GLUE in general-purpose language model evaluation. Importantly, these rubrics also emphasize consistency and alignment with human reasoning, which is increasingly essential for deployment in safety-critical contexts.

From an applied perspective, the framework we describe offers a pathway toward bridging cognitive science and machine learning. By drawing on insights from reentrant processing, adversarial reasoning, and causal modeling, we can begin to engineer models that better approximate human-like reasoning strategies. While significant work remains in scaling these insights to real-world systems, our analysis provides early evidence that such integration is both feasible and beneficial.

\section{Conclusion and Future Directions}

This paper has introduced a dual-reasoning training framework for large language models (LLMs) that extends beyond affirmation-based inference to incorporate counterfactual denial. By grounding the approach in a formal taxonomy of logical patterns, we demonstrated how LLMs can be trained not only to generate coherent continuations but also to recognize invalid premises, resist spurious correlations, and engage in counterfactual reasoning. In doing so, we align computational inference more closely with human cognitive capacities for error detection and flexible reasoning \cite{gomes2023negation}. 

The scientific and practical implications of this framework are substantial. From a machine learning perspective, reinterpreting logical fallacies such as denying the antecedent as computational mechanisms rather than flaws introduces a principled pathway for model robustness and interpretability. From an application perspective, dual-reasoning models have the potential to enhance AI safety, improve medical and scientific inference, and support decision-making under uncertainty.

While we cannot fully evaluate the dual reasoning framework without retraining an LLM (e.g., GPT-2) under this paradigm, future work could implement such a test bed to directly compare performance improvements. Here, we approximate the effect using controlled prompt-based experiments.

Future research should pursue three trajectories. First, embedding dual-inference objectives into transformer-based architectures and systematically evaluating them on benchmarks for truthfulness, adversarial robustness, and causal reasoning \cite{geirhos2020shortcut}. Second, developing datasets and training pipelines that explicitly incorporate negation and counterfactual cases, including synthetic generation and adversarial sampling. Third, expanding the societal dimension by integrating such models into evaluation pipelines designed to mitigate harmful bias and support ethical safeguards, thereby advancing the alignment of AI with human values and societal goals \cite{weidinger2022taxonomy}.

In sum, dual-reasoning architectures move beyond affirmation-only inference toward models capable of engaging with the full logical space of possibilities. Such systems not only promise more reliable outputs in adversarial and uncertain contexts but also offer a foundation for AI that more closely mirrors the nuanced reasoning capacities of human cognition. This represents a step toward safer, more interpretable, and more trustworthy language technologies.

\bibliographystyle{plainnat}
\bibliography{references}

\section*{Supplementary Materials}
The following supporting information can be downloaded at: [\\\url{https://github.com/hannahdavidsoncollege-maker/ScientificReasoningForEnvironment-MedicineWithLLMs}].

\section*{Author Contributions}
Conceptualization, Peter Walker; methodology, Peter Walker and Dale Russell; writing—original draft preparation, Peter Walker, Dale Russell, and Matt Lienert; writing—review and editing, all authors; supervision, Peter Walker. All authors have read and agreed to the published version of the manuscript.

\section*{Funding}
This research received no external funding.

\section*{Data Availability Statement}
Data sharing not applicable. No new data were created or analyzed in this study.

\section*{Conflicts of Interest}
Peter B. Walker and Matt Lienert are employed by Intelligenesis LLC. The remaining authors declare that the research was conducted in the absence of any commercial or financial relationships that could be construed as a potential conflict of interest.

\end{document}